%% file: main.tex
\documentclass[nonacm,sigconf]{acmart}
\renewcommand\footnotetextcopyrightpermission[1]{} 
\newtheorem{theorem}{Theorem}[section]

\newtheorem{definition}[theorem]{Definition}

\newtheorem{example}{Example}[section]

\input{math_commands.tex}

\usepackage{comment}
\usepackage{amsmath,amsfonts}
\usepackage[ruled,vlined,linesnumbered]{algorithm2e}

\usepackage{hyperref}
\usepackage{url}
\usepackage{todonotes}
\usepackage{multirow}
\usepackage{booktabs}

\begin{document}

\title{Certified Continual Learning for Neural Network Regression}
%
%
\author[Long H. Pham \& Jun Sun]{Long H. Pham \& Jun Sun \\
School of Computing and Information Systems, Singapore Management University \\
Singapore, Singapore\\
\texttt{\{hlpham,junsun\}@smu.edu.sg}
}

\newcommand{\fix}{\marginpar{FIX}}
\newcommand{\new}{\marginpar{NEW}}
\newcommand{\notimplies}{\;\not\!\!\!\implies}



\input{contents/0_abstract}

\maketitle              

\input{contents/1_intro}

\input{contents/2_background}
\input{contents/3_prelim}
\input{contents/4_approach}
\input{contents/5_eval}
\input{contents/6_related}
\input{contents/7_conclusion}
%

%
%
%
\bibliography{ref}
\bibliographystyle{ACM-Reference-Format}


%
\end{document}

%% file: math_commands.tex
\usepackage{amsmath,amsfonts,bm}

\def\eqref#1{equation~\ref{#1}}

\def\1{\bm{1}}

\DeclareMathAlphabet{\mathsfit}{\encodingdefault}{\sfdefault}{m}{sl}
\SetMathAlphabet{\mathsfit}{bold}{\encodingdefault}{\sfdefault}{bx}{n}

\DeclareMathOperator*{\argmin}{arg\,min}

%% file: contents/0_abstract.tex
\begin{abstract}
On the one hand, there has been considerable progress on neural network verification in recent years, which makes certifying neural networks a possibility. On the other hand, neural networks in practice are often re-trained over time to cope with new data distribution or for solving different tasks (a.k.a.~continual learning). Once re-trained, the verified correctness of the neural network is likely broken, particularly in the presence of the phenomenon known as catastrophic forgetting. In this work, we propose an approach called certified continual learning which improves existing continual learning methods by preserving, as long as possible, the established correctness properties of a verified network. Our approach is evaluated with multiple neural networks and on two different continual learning methods. The results show that our approach is efficient and the trained models preserve their certified correctness and often maintain high utility.  
\end{abstract}

%% file: contents/1_intro.tex
\section{Introduction}
Neural network based deep learning is increasingly applied in safety/security-critical systems, such as self-driving cars~\cite{bojarski2016end}, and face recognition for authentication~\cite{parkhi2015deep}. Thus, methods and tools that are capable of verifying their correctness (with respect to a range of properties such as safety~\cite{katz2017reluplex}, robustness~\cite{singh2019abstract}, backdoor-freeness~\cite{pham2022backdoor}, fairness~\cite{sun2021probabilistic} and privacy~\cite{DBLP:conf/ccs/AbadiCGMMT016}) are highly desirable. In recent years, there has been exciting development in neural network verification techniques and tools~\cite{katz2017reluplex,singh2019abstract,gehr2018ai2}, which has made verifying and certifying (small to medium size as of now) neural networks possible. 

Like traditional software systems, neural networks often evolve over time, i.e., they may be updated through retraining (a.k.a.~fine-tuning) to cope with the latest data distribution or to perform adjusted tasks. For instance, a recent trend is to build neural networks by retraining existing general purpose neural network models (such as GPT-3~\cite{brown2020language}). In fact, there is a whole area of research on this topic called continual learning~\cite{van2019three}. One of the central problems to be addressed by continual learning techniques is the phenomenon known as catastrophic forgetting~\cite{mccloskey1989catastrophic,ratcliff1990connectionist}, i.e., retraining with a different data distribution would result in a dramatic change in the trained behavior of the neural network. 

For neural networks used in safety-critical systems, catastrophic forgetting is particularly concerning. That is, once a verified (against certain critical properties) neural network is re-trained, there is no guarantee that those critical properties are still satisfied. In fact, the properties are unlikely satisfied anymore if catastrophic forgetting occurs and it is highly non-trivial to repair the neural network.
It thus would be desirable if we can apply continual learning in a way such that the verified properties are preserved as much as possible. To the best of our knowledge, how to conduct continual learning without compromising the verified correctness of a neural network is an unexplored topic. 

In this work, we propose an approach called certified continual learning (CCL) to address the problem. {\color{black} Our main goal is to develop a training procedure that satisfies three properties: (i) the trained model should have high accuracy on both new and old learning tasks, (ii) the trained model should preserve as many verified properties as possible, (iii) the training procedure should not introduce a lot of overhead compared to the ordinary retraining procedure. We remark that our approach is particularly useful in case the users want their models to preserve some predefined properties (e.g., robustness or fairness) while adapting to new learning tasks.}

Our approach starts with a proposal that once a neural network is verified, a certificate for each verified property should be attached to the neural network. The certificates provide a way for us to apply continual learning techniques with multiple additional steps such that the correctness of the neural network is likely preserved. First, we propose certificate-based data synthesis, which generates additional training data for the retraining so that the network is ``reminded'' of the certificates. This step is extremely lightweight (i.e., sampling according to the precondition of the certificates) and yet effective. The second is a certificate regularizer, which penalizes modifying the neural network parameters beyond what is allowed according to the certificates. The third is a model clipping step that re-establishes the certificate if it is found to be broken, i.e., by relaxing the certificates locally based on Craig interpolation if feasible and selectively adjusting the neural network parameters so that the correctness of the neural network is preserved. Note that CCL is designed to achieve its goal with lightweight techniques first before applying heavy formal analysis techniques. Furthermore, CCL works with different kinds of continual learning methods, i.e., those that introduce new neurons~\cite{rusu2016progressive,li2017learning} and those that rely solely on existing neurons~\cite{kirkpatrick2017overcoming}.

To evaluate the feasibility of CCL, we conduct multiple experiments with the following setting. We use DeepPoly~\cite{singh2019abstract}, an existing state-of-the-art abstraction domain and tool to construct the certificates for multiple neural networks with respect to reachability, (local) robustness, and (individual) fairness properties. The certificates are in the form of concrete lower and upper bounds for each neuron. Results on models trained on the ACAS Xu dataset, the MNIST dataset, the CIFAR dataset, and the Census dataset (with and without introducing new neurons) show that CCL effectively maintains the verified correctness of the neural network models. To summarize, using CCL we can maintain 100\% of the verified correctness properties on all the models trained on these four datasets, whereas existing continual learning often fails to preserve the verified properties.

The rest of the paper is organized as follows. In Section~\ref{sec:background}, we present relevant background. In Section~\ref{sec:prelim}, we define our problem. In Section~\ref{sec:approach}, we present details of our approach. In Section~\ref{sec:eval}, we present the experimental evaluation results. In Section~\ref{sec:related}, we discuss related work. Lastly, we conclude in Section~\ref{sec:conclusion}.

%% file: contents/2_background.tex
\section{Background} \label{sec:background}

\subsection{Neural Network Verification} \label{nnverify}

For neural network verification, the network is considered a function, which maps an input vector to an output vector. A property can be typically specified in form of $pre \implies post$ where $pre$ is a precondition and $post$ is a postcondition. These conditions are typically in the form of first-order logic formulae, which capture the constraints on the input vector and output vector respectively. The following are a few popular neural network properties that are supported by our approach. 

\begin{itemize}
    \item A \emph{reachability} property specifies that if an input vector satisfies a certain constraint, the output vector must satisfy a certain constraint. For instance, the following is property 3 for a model trained on the ACAS Xu dataset~\cite{katz2017reluplex}, which is a neural network model for airborne collision avoidance. 
    $$
    \arraycolsep=2.5pt
    \begin{array}{lll}
        pre & \equiv & -0.3035 \leq x[0] \leq -0.2986 \wedge -0.0095 \leq x[1] \leq 0.0095 ~\wedge \\ 
    & & 0.4934 \leq x[2] \leq 0.5 \wedge 0.3 \leq x[3] \leq 0.5 \wedge 0.3 \leq x[4] \leq 0.5 \\
    post & \equiv & L(y) \neq 0 
    \end{array}$$
    where $x$ is the input vector; $y$ is the corresponding output vector; and $L$ is the labeling function (i.e., $L(y)$ returns the index of the maximum value in the output vector $y$). 
    Intuitively, this property specifies that under certain conditions, the network should never advise the first label.
    
    \item \emph{Robustness} property means that, given any sample $x_0$, any sample $x$ which is similar to $x_0$ must have the same label as $x_0$. Formally, it is specified as follows.
    $$\begin{array}{lll}
        pre & \equiv & d(x,x_0) \leq c \\
        post & \equiv & L(y) = L(y_0)
    \end{array}$$
    where $x_0$ is a constant vector representing a specific sample; $x$ is any other input vector; $y_0$ and $y$ are the output vectors corresponding to $x_0$ and $x$ respectively; $c$ is a constant number; and $d$ is a predefined function which measures the distance between two input vectors. 
    
    \item \emph{Fairness} is a desirable property for neural network models which have societal impact. In this work, we consider the property called individual fairness (a.k.a.~counterfactual fairness). Intuitively, the property states that for all pairs of input samples which differ only by certain sensitive features (such as gender or race), their labels must be the same. For instance, assume that the first feature of an input vector is the only sensitive feature (such as gender), such a fairness property can be formalized as follows.
    $$
    \arraycolsep=2.5pt
    \begin{array}{lll}
       pre & \equiv & x_1[0] \neq x_2[0] \wedge x_1[1]=x_2[1] \wedge \cdots \wedge x_1[n]=x_2[n] \\
       post & \equiv & L(y_1) = L(y_2) 
    \end{array}$$
\end{itemize}

To verify the above-mentioned properties and beyond, existing neural network verification approaches can be roughly divided into two categories, i.e., precise methods and approximation-based methods. The former precisely encodes the network and the property into some form of constraint-solving problem (e.g., SMT~\cite{katz2017reluplex,ehlers2017formal} and MILP~\cite{tjeng2017evaluating}). Thanks to precise encoding, these approaches typically guarantee soundness and completeness. However, they often suffer from scalability problems. Approximation-based methods build an abstract state from the precondition, then propagate the abstract state through the network using techniques such as abstract interpretation~\cite{gehr2018ai2,singh2019abstract}. The abstract state of the output layer is then used to verify the property. In general, approximation-based methods are sound (but not complete due to the over-approximation) and more efficient. 
Furthermore, many abstraction domains have been explored in these approximation-based methods, such as star-sets~\cite{tran2019star,tran2020verification} and DeepPoly~\cite{singh2019abstract}.

\subsection{Continual Learning} 
Continual learning is an area of research that is motivated by the need of continuously updating a trained neural network. One such scenario is that the model is required to learn from a stream of data produced by a non-stationary, dynamic environment. Since the data distribution may drift over time, the i.i.d.~(independent and identically distributed) assumption behind traditional machine learning training procedures may be violated, causing problems such as catastrophic forgetting of previous knowledge~\cite{mccloskey1989catastrophic,ratcliff1990connectionist}. Note that a similar situation arises in the setting of the so-called incremental-learning~\cite{shmelkov2017incremental}, online learning~\cite{ritter2018online}, or even reinforcement learning~\cite{rolnick2019experience,atkinson2021pseudo}.

Formally, we can assume that there is a stream of data $\langle (D_1, T_1),$ $(D_2, T_2), ..., (D_n, T_n) \rangle$, in which each $D_i$ is a training set and each $T_i$ is a test set. Let $N_1$ denote the neural network trained on $D_1$ in a conventional way. The goal of continual learning is at step $t$, retraining the network $N_{t-1}$ with the new training set $D_t$ such that it performs well on $T_t$ and avoids catastrophic forgetting, i.e., $N_t$ maintains its performance on $T_i$ with $ 1 \leq i \leq t-1$. Note that depending on the application, the new data may represent a slightly shifted distribution of the data or a completely different distribution (so that the neural network is retrained to perform new tasks). 

Note that it is typically assumed that it is infeasible to keep all the data $D_1, D_2, \cdots, D_{t-1}$ and train $D_t$ from scratch. For instance, the model could be sourced from a third-party such as from an open market or by a service provider. In some settings, we may assume the users can collect a small set of data that follows a distribution that is similar to the original training dataset, which is an assumption adapted in some of the existing continual learning methods~\cite{shin2017continual,van2018generative,rebuffi2017icarl,nguyen2017variational,kemker2017fearnet}.

Continual learning is relevant in safety-critical applications such as cyber-physical systems (e.g., for anomaly detection~\cite{doshi2020continual}). If the neural network must satisfy certain critical properties (e.g., reachability, robustness, and fairness), catastrophic forgetting is particularly worrisome as it means that the neural network is likely to violate the property along the way even if it is initially verified. In the literature, existing continual learning methods typically aim to address the problem of catastrophic forgetting by using regularizers~\cite{li2017learning,kirkpatrick2017overcoming}, introducing new neurons~\cite{rusu2016progressive,li2017learning}, or replaying old training data~\cite{shin2017continual,van2018generative,rebuffi2017icarl,nguyen2017variational,kemker2017fearnet}.

While the above-mentioned approaches may sometimes be empirically effective with respect to the accuracy, none of them focus on maintaining verified neural networks' properties through one or more retraining. 

%% file: contents/3_prelim.tex
\section{Problem Definition} \label{sec:prelim}
In this section, we define our problem. Our overall goal is to develop a continual learning approach such that a verified neural network is likely to remain correct.  Our high-level idea is to, given a neural network, maintain a certificate for each verified property and use the certificates to guide the retraining of the neural network.

\begin{definition}[Neural Networks]
A neural network is a function $N: \mathbb{R}^{p} \to \mathbb{R}^{q}$ which maps an input $x\in \mathbb{R}^{p}$ to an output $y \in \mathbb{R}^{q}$. It is composed of multiple layers, written as $N = f^{k-1} \circ f^{k-2} \circ \cdots \circ f^0$ where $f^l$ is the $l$-th layer (which could be an affine layer, a pooling layer, or an activation layer, e.g., ReLU, Sigmoid, and Tanh). Each layer contains multiple neurons. 
\end{definition}
Note that we assume that affline layers (i.e., a weighted sum plus certain bias) and activation layers are separated for the ease of discussion as in~\cite{singh2019abstract}. We write $x^{l+1}$ to denote the variables representing the output of those neurons at layer-$l$. Note that $x^0$ represents the input of the network. 

\begin{example}
    Hereafter, we use one ACAS Xu neural network adopted from~\cite{katz2017reluplex} and one neural network trained on the MNIST dataset as our running examples. The former is used as an airborne collision avoidance system, which generates one among five action advisories, i.e., no conflict, turn left weakly, turn right weakly, turn left strongly, and turn right strongly. It has five inputs, six hidden layers, each of which contains 50 neurons. The latter is used to classify hand-written digits. It has 784 inputs, three hidden layers, each of which contains 10 neurons. 
\end{example}

Our approach starts with a proposal to attach certificates with a verified neural network, one for each verified property. Formally, we define a certificate as follows in this work.

\begin{definition}[Neural Network Certificate]
Given a neural network $N = f^{k-1} \circ f^{k-2} \circ \cdots \circ f^0$, a neural network certificate $\Phi$ with respect to a property $\xi$ is a sequence $\langle \phi^0, \phi^1, \cdots, \phi^k \rangle$ where each $\phi^l$ is constraint over neurons at layer-$l$, i.e., $x^l$, satisfying the following conditions.
\begin{itemize}
    \item $\phi^l \land f^l(x^l) \implies \phi^{l+1}$ for all $l$;
    \item and $\phi^k \implies \xi$.
\end{itemize}
\label{def:cert}
\end{definition}
Note that $\phi^0$ can be used to encode the precondition of a property, i.e., the constraint must be satisfied by the input. Similar to the idea of using proof transfer to speed up the verification~\cite{ugare2022proof}, we require that the certificate includes the result of intermediate verification steps so that an independent third party can efficiently validate the certificate by checking whether $\phi^l$ for all $l$ is satisfied locally (rather than verifying the neural network again). More relevantly (to this work), such certificates allow us to guide the continual learning processes in multiple ways. 

Intuitively, each certificate captures the over-approximation of the network's layers according to each property. To be useful in our approach, we require that the certificates should not be too weak, e.g., $True$. We remark that such certificates can be readily generated using existing neural network verifiers such as AI$^2$~\cite{gehr2018ai2}, DeepPoly~\cite{singh2019abstract}, RefinePoly~\cite{singh2019beyond}, and $\alpha$,$\beta$-CROWN~\cite{zhang2018efficient,xu2021fast,wang2021beta,zhang2022general}. Furthermore, depending on the particular verification techniques used by the verifier, the constraint $\phi^l$ may take different forms, e.g., interval constraints or polyhedral constraints. 

\begin{example}
    A total of 10 safety properties for the ACAS Xu neural network are specified in~\cite{katz2017reluplex}. We focus on property 3 as described in Section~\ref{nnverify}. We use DeepPoly to verify property 3 and obtain the certificate, where $\phi^i$ is in the form of a concrete range for each neuron. A part of the certificate is shown below.
    \begin{align*}
    \phi^{13} & \equiv -3.3497 \leq x^{13}[0] \leq 1.7532 \wedge -5.6635 \leq x^{13}[1] \leq -4.7525 ~\wedge \\
    & 2.0411 \leq x^{13}[2] \leq 2.3862 \wedge 0.4508 \leq x^{13}[3] \leq 1.9311 ~\wedge \\
    & 1.1933 \leq x^{13}[4] \leq 4.2636
    \end{align*}
    where $\phi^{13}$ is the constraint on the output layer, based on which we can infer that the prediction is not 0. For the MNIST model, the initial model is trained to recognize digits 0 and 1 only. We use DeepPoly to verify local robustness based on 25 selected samples (of 0 and 1) with respect to a distance function $d$ defined based on $L_{\infty}$-norm with a threshold $\epsilon = 0.01$.  
    We thus have a total of 25 certificates for this network. For instance, the certificate associated with the first selected sample with label 1 with respect to the output layer is as follows.
    $$
    \phi^7 \equiv -3.6951 \leq x^7[0] \leq -3.4523 \wedge 2.5819 \leq x^7[1] \leq 2.7624
    $$
    Based on $\phi^7$, we can infer $x^7[1] > x^7[0]$, which means that the output is 1.
\end{example}

We are now ready to define our problem. Given a neural network $N$ with multiple certificates $\{\Phi_i\}$, assuming that we would like to retrain $N$ (i.e., continual learning), the problem is how to retrain $N$ such that the updated $N$ is likely to preserve the certificates $\{\Phi_i\}$. Furthermore, our method should support both ways of continual learning, i.e., one without introducing new neurons~\cite{kirkpatrick2017overcoming} and one with new neurons~\cite{rusu2016progressive,li2017learning}. We notice that our problem can be considered as the generalization of suppressing regression errors in neural networks \cite{tokui2022neurecover,you2023regression,li2023lightweight} because we want to preserve general properties instead of the correctness of individual samples.

\begin{example}
For the ACAS Xu model, we retrain it with samples generated based on the precondition of property 10. The idea is to further train the neural network so that it is capable of generating correct predictions under different circumstances, as well as to preserve property 3 after retraining (so that no neural network repairing is needed). Note that in this setting, we assume that no new neurons are introduced. 

The MNIST model is retrained with images of 2 and 3. Our goal is to update the network with the new data so it can accurately recognize 0, 1, 2, and 3, and remains robust with respect to the 25 selected samples.
Note that for this network, we assume that 10 new neurons are introduced for each hidden layer and two new neurons are introduced for the output layer.
\end{example}

%% file: contents/4_approach.tex
\section{Our Approach} \label{sec:approach}
In this section, we present the details of our approach. The overall workflow of CCL is shown in Algorithm~\ref{alg:overall}. The inputs are a verified neural network $N$, its verification certificates $\{ \Phi_i \}$, and a set of new data $\{ (x_j^0, y_j) \}$. The output is a new model $M$ which is expected to be the result of retraining $N$ with the new data. \\
 
\begin{algorithm}[t]
    Inputs: a certified neural network $N$; certificates $\{\Phi_i\}$; and new training samples $\{ (x_j^0, y_j) \}$ \\
    Initialize $M$ based on $N$; \\
    Augment $ \{ (x_j^0, y_j) \}$ with synthesized data from $N$ and $\{\Phi_i\}$; \# \emph{data augmentation} \\
    \For{each epoch of a batch of new data} {
        Train $M$ with the objective: $L_{ce}(\theta, \{ (x_j^0, y_j) \}) + \alpha Reg(\theta, \{ (x_j^0, y_j) \}, \{\Phi_i\})$; \# \emph{regularization} \\
    }
    Clip $M$'s neuron weights according to $\{\Phi_i\}$ using Algorithm~\ref{alg:modelclipping}; \# \emph{model clipping} 
\caption{CCL}
\label{alg:overall}
\end{algorithm}

\noindent \textbf{Model Initialization} We start with initializing the new model $M$ based on the given certified model $N$. If no new neurons are introduced, $M$ is initialized to be identical to $N$. Otherwise, we initialize $M$ by copying all neurons from $N$ and add additional neurons to each layer of $N$ (except the input layer if the dimension of the input remains the same). The newly introduced neurons are initialized randomly. Note that these new neurons are not associated with any certificate.

\begin{example}
    Figure~\ref{fig:network} illustrates the initialized neural network for the MNIST model. On the left, we have the old model which is trained to recognize 0 and 1. The input layer has 784 neurons and each hidden layer has 10 neurons. On the right, we have a new model which is trained to recognize four digits 0, 1, 2, and 3. Each hidden layer has 10 more neurons and the output layer has two more neurons. The blue circles and arrows represent the old neurons and weights, which are copied from the old network. The red circles and arrows show the new neurons and weights, which are initialized randomly. We connect old neurons with new neurons so that the new neurons can make use of the features that are learned by the old neurons. We do not connect new neurons to the old neurons so that the new features that are learned by the new neurons do not affect the old neurons. As shown in~\cite{rusu2016progressive}, such a way of continual learning reduces the effect of catastrophic forgetting. 
\end{example}

\begin{figure}[t]
\centering
\includegraphics[scale=0.055]{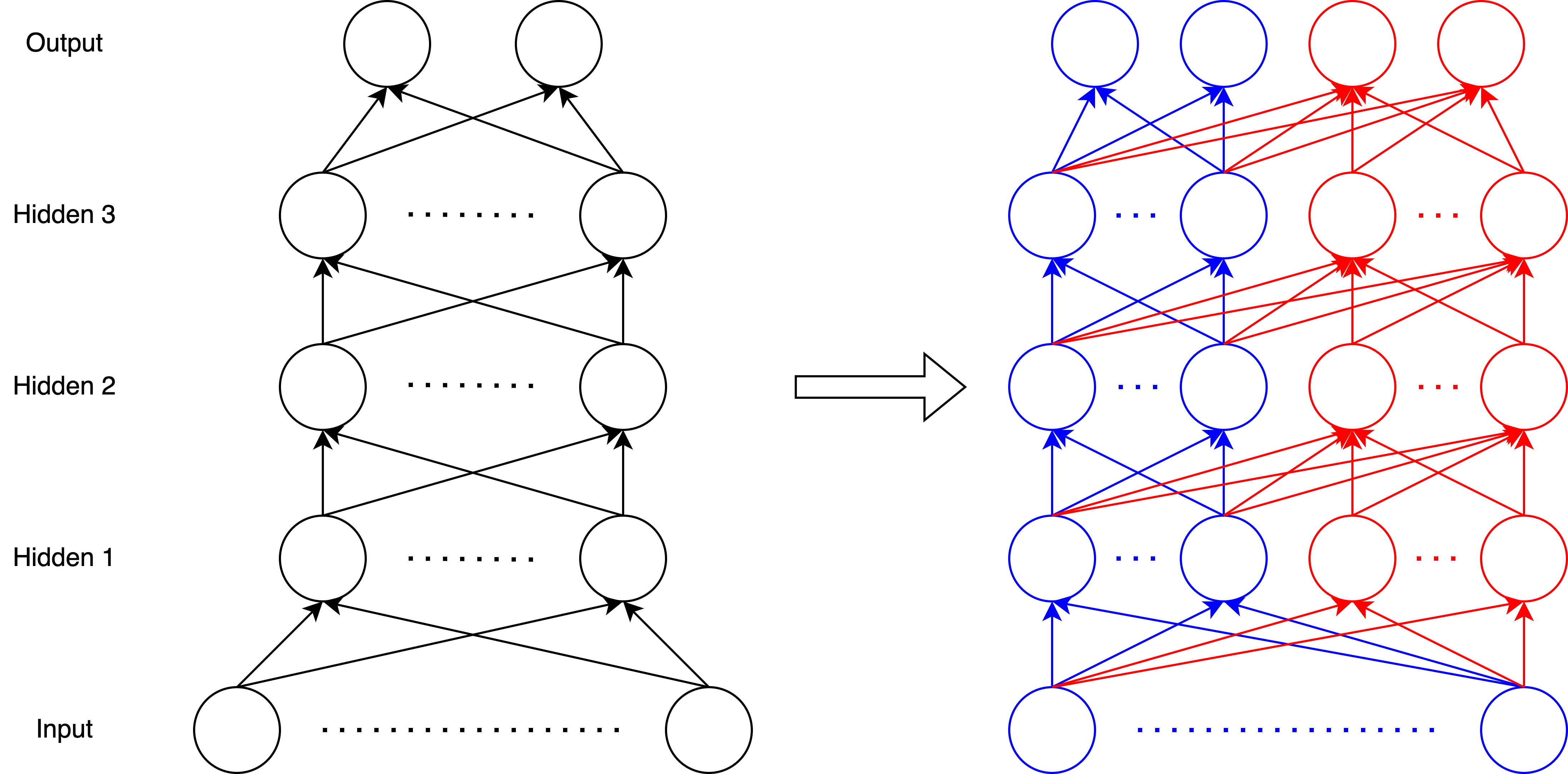}
\caption{An example of model initialization}
\label{fig:network}
\end{figure}

\noindent \textbf{Certificate-based Data Augmentation} Next, we have a lightweight step that generates some data based on $N$ and $\{\Phi_i\}$ to enrich $\{ (x_j^0, y_j) \}$. This aims to reduce the effect of catastrophic forgetting with regard to verified properties. Note that this step is particularly useful if we do not have any of the original training data, and is shown to be useful even if part of the original training data is available (refer to Section~\ref{sec:eval}). This step generates the data as follows. For each certificate $\Phi_i$, we sample uniformly according to the constraint $\Phi_i^0$ (which constraints $x^0$). Each input is then fed into $N$ to collect the labels according to $N$. 

\begin{example}
    For our running example, it works as follows. For the ACAS Xu network, we use $\phi^0$ of property 3 to generate 5 samples, which are then fed into the old network to collect the labels before adding them into the new training data. For the MNIST network, for each certificate (for the local robustness with respect to a selected sample $x$), we generate 10 samples according to the corresponding $\phi^0$. By definition of local robustness, these samples all have the same label as $x$. They are then used to enrich the new training data.  
\end{example}

We remark that while certificate-based data augmentation is helpful in reducing the effect of catastrophic forgetting (and consequently in maintaining the certificates), it must be applied with caution. This is because the data synthesized based on the certificates are likely to follow a distribution that is different from the original training data, and might alter the behavior of the neural network unexpectedly. We thus keep the number of synthesized samples small, e.g., no more than 10 for each certificate. \\ 

\noindent \textbf{Certification-based Regularization} Next, CCL iteratively trains $M$ epoch-by-epoch with the standard cross-entropy loss as well as a \emph{certificate-based regularizer}. The optimization problem is defined as follows.
\[
\argmin_{\theta}~L_{ce}(\theta,\{ (x_j^0, y_j) \}) + \alpha * Reg(\theta, \{ (x_j^0, y_j) \}, \{\Phi_i\})
\]
where $L_{ce}(\theta,\{ (x_j^0, y_j) \})$ is the standard cross-entropy loss, $\alpha$ is a constant hyperparameter, and $Reg(\theta, \{ (x_j^0, y_j) \}, \{\Phi_i\})$ is the certi-ficate-based regularizer. Intuitively, the certificate-based regularizer is used to discourage updating the weights of neurons that are associated with the certificates so that the certificates are unlikely to break. The certificate-based regularizer is defined as follows. 
\[
Reg(\theta, \{ (x_j^0, y_j) \}, \{\Phi_i\}) = \Sigma_i \Sigma_j \rho(x_j^0, y_j, \Phi_i)
\]
where $\rho(x_j^0, y_j, \Phi_i)$ is defined as follows.
\begin{equation*}
  \rho(x_j^0, y_j, \Phi_i) =
    \begin{cases}
      0 & \text{if $x_j^0 \notimplies \Phi_i^0$}\\
      \Sigma_{l=1}^k (dist(f^{l-1} \circ \cdots f^0(x_j^0), \phi_i^l)) & \text{otherwise}
    \end{cases}       
\end{equation*}
where $dist(v, \varphi)$ is a metric for measuring how far a variable valuation $v$ is from satisfying a constraint $\varphi$. Intuitively, $\rho(x_j^0, y_j, \Phi_i)$ is defined such that if the input $x_j^0$ is irrelevant to the certificate $\Phi_i$ (i.e., not satisfying its precondition), then its value is 0 (i.e., no penalty); otherwise, we evaluate how much the constraints from the certificate are violated layer-by-layer and use the sum as the penalty. 

Note that such a distance $dist(v, \varphi)$ can be defined in different ways
based on what is the form of constraint $\phi_i$ used in the certificate, i.e., the abstraction domain. For instance, if $\phi_i$ is in the form of concrete intervals $[lw, up] $ for each neuron, we can define the distance as follows (so that the training can be done efficiently).
\begin{equation*}
  dist(v, [lw, up]) =
    \begin{cases}
      (v - ((lw + up) / 2))^2 & \text{if $v < lw$ or $v > up$}\\
      0 & \text{otherwise}
    \end{cases}       
\end{equation*}

In the case that $M$ contains newly introduced neurons, we add a further regularizer so as to discourage changing the existing neurons altogether. Let $\theta_0$ be the parameters of $N$ before the retraining. The following is the optimization problem to be solved through training. 
\begin{align*}
 \argmin_{\theta} L_{ce}(\theta,\{ (x_j^0, y_j) \}) ~+~ & \alpha * Reg(\theta, \{ (x_j^0, y_j) \}, \{\Phi_i\}) \\
  +~ & \beta * \Sigma_{p \in \theta_0}(\theta_0(p) - \theta(p))^2
\end{align*}
where $\beta$ is another constant hyper-parameter, $p$ is any parameter of the old neural network and $\theta(p)$ denotes the value of parameter $p$ given the current neural network parameter $\theta$. The intuitive idea is that this way the old neurons are encouraged to remain unchanged whereas the new neurons are responsible for the new task.
We remark that simply demanding that the old neurons cannot be modified at all is not a good idea as those neurons may capture out-of-dated features and need to be lightly tuned to better capture new features in the new training data~\cite{rusu2016progressive}. \\

\noindent \textbf{Certificate-based Model Clipping} Note that the regularization does not guarantee that the old neurons are not modified or even modified beyond the certificates. Thus, CCL has a last step to check whether the certificates remain valid.
Note that checking whether a certificate is valid can be done efficiently, i.e., by checking whether $\phi^l \land f^l(x^l) \implies \phi^{l+1}$ is satisfied for all $l$ locally.
If the certificate is not valid, we re-verify the corresponding property. If the property is successfully re-verified, we update the certificate accordingly.
Otherwise, we try to restore the property by `clipping' the relevant neurons' weights according to the old certificate as follows.

If the certificate is broken, ideally we would like to repair the neural network by solving the following optimization problem. 
$$\argmin_{\theta'} \Sigma_{p \in \theta} (\theta(p) - \theta'(p))^2$$ such that $$\forall i,l.~\phi_i^l \land f^l(x^l) \implies \phi_i^{l+1}$$
where $\theta'$ is the new network parameters. That is, we would like to modify the neural network parameters minimally such that the certificates are restored. In practice, solving this (global) optimization problem is challenging and often impractical. However, since we assume that the certificates are in form of constraints for each neuron, we can solve this problem locally, i.e., by minimally modifying the neurons whose constraints according to the certificate are broken so that the certificate becomes valid. Furthermore, this optimization problem can often be simplified based on the abstraction domain used for $\phi_i$. In the following, we show how we solve the problem if $\phi_i$ is in the form of concrete intervals. 

\begin{algorithm}[t]
    Inputs: a retrained neural network $N$; certificates $\{\Phi_i\}$; \\
    Let $newCert$ be $\{\}$; \\
    \For{each certificate $\Phi_i$ in $\{\Phi_i\}$} {
        Add $\Phi_i$ to $newCert$; \\
        \For{each neuron $x$ of an affine layer in $N$ such that  $x = w * x_i + b$} {
            Let $[lw_j, up_j]$ be the constraint associated with $x$ with respect to each $\Phi_j$ in $newCert$; \\
            Let $max_{lw}$ be $max(lw_j - min(w * x_i))$; \\
            Let $min_{up}$ be $min(up_j - max(w * x_i))$; \\
            \If{$max_{lw} < min_{up}$} {
                Set the bias $b$ to be $(max_{lw} + min_{up}) / 2$; \\    
            }
            \Else {
                Remove $\Phi_i$ from $newCert$; \\
                Revert the biases' changes; \\
                Break; \\
            }
        }
    }
    Let $newCert$ be the new certificates;
\caption{Model clipping for interval constraints}
\label{alg:modelclipping}
\end{algorithm}

Algorithm~\ref{alg:modelclipping} shows how we solve the problem by focusing neurons of the affine layers. Note that for many neural networks, most of the parameters in $\theta$ are associated with the affine layers in the form of weights and biases. There may be additional parameters such as the parameter of a parametric ReLU layer in some neural networks. Although they are not the focus of the discussion hereafter, they can be supported in principle by our approach. 

For simplicity, in the following, we suppose there is only one property and one certificate (Algorithm~\ref{alg:modelclipping} shows the general case with multiple properties and multiple certificates). Consider a neuron of an affine layer $x = w * x_i + b$ where $w$ is a vector representing the weights of the neuron, $b$ is the bias, and $x_i$ are the vector of neurons of the previous layer. Assume that the certificate associated with $x$ is $[lw, up]$ which is violated. Our problem is thus to modify either $w$ or $b$ such that $lw \leq v = w * x_i + b \leq up$. For simplicity, assume that we are to focus on modifying $b$. The new value of $b$ must satisfy $lw - min(w * x_i) \leq b \leq up - max(w * x_i)$. Note that $min(w * x_i)$ and $max(w * x_i)$ can be easily computed based on the concrete values of $w$ and the concrete intervals of $x_i$ (which is a part of the certificate). 

Based on the above discussion, the model clipping is performed as follows. First, we find a neuron $x$ whose constraint according to the certificate is violated. We then compute $lw - min(w * x_i)$ and $up - max(w * x_i)$ for the neuron. If $lw - min(w * x_i) < up - max(w * x_i)$, we set the bias of the neuron to be $(lw - min(w * x_i) + up - max(w * x_i)) / 2$, which guarantees that the neuron satisfies its certificate. Note that we can similarly fix the neuron by modifying $w$. After fixing the neuron, we then check whether all the certificates are valid. If it is not the case, we repeat the above process until all parts of the certificates are satisfied. It is straightforward to see that if we are successful in ``fixing" all the neurons, the neural network is guaranteed to satisfy the verified properties. In the case of Algorithm~\ref{alg:modelclipping}, any certificate that remains is guaranteed to be a valid certificate.  

We remark that the quality of the certificates makes a difference in the above-mentioned approach, i.e., the more relaxed the certificate is (e.g., the larger the interval $[lw, up]$ is), the more likely that we can fix the neurons locally. 
In other words, if the certificate is very tight, the above-mentioned model clipping approach may fail to find a solution.
In the following, we present an approach to relax the certificate so that the above-mentioned model clipping is more likely to succeed. \\ 

\noindent \textbf{Certificate Relaxation} The problem of obtaining a relaxed condition which guarantees certain postcondition is satisfied is a known problem in the system/software verification community, and goes by the name of weakest-precondition computation in Hoare Logic~\cite{hoare1969axiomatic} or Craig interpolation~\cite{craig1957three}. In this work, we apply Craig interpolation to obtain a relaxed certificate for the neurons to be fixed before applying the above-mentioned model clipping approach. Formally, 

\begin{definition}[Craig interpolation]
Let $P$ and $Q$ be two first-order formulae such that $P \implies Q$. A Craig interpolant for $P$ and $Q$ is a formula $C$ that satisfies the following conditions:
\begin{itemize}
    \item $P \implies C$ and $C \implies Q$.
    \item All of the variables in $C$ also appear in both $P$ and $Q$.
\end{itemize}
\end{definition}


In the following, we explain how we can apply Craig interpolation to relax a certificate systematically. Based on Definition~\ref{def:cert}, we can deduce the following, for each property $\xi$.
$$
\bigwedge_{l=0}^{k-1}{(\phi^l \land f^l(x^l) \implies \phi^{l+1})}\land \phi^k \implies \xi
$$
Thus, we can systematically relax each $\phi^{l+1}$ one-by-one, starting with $\phi^k$. That is, since $\phi^k \implies \xi$, we compute the Craig interpolant for $\phi^k$ and $\xi$, which we denote as $\varphi^k$. Since $\phi^{k-1} \land f^{k-1}(x^{k-1})  \implies \phi^k \implies \varphi^k$, we then compute the Craig interpolant for $\phi^{k-1} \land f^{k-1}(x^{k-1})$ and $\varphi^k$, which we denote as $\varphi^{k-1}$. Note that in general $f^{k-1}(x^{k-1})$ can be a complicated non-linear function. We thus apply existing abstraction techniques such as DeepPoly to construct an over-approximation of $f^{k-1}(x^{k-1})$, denoted as $O_f^{k-1}(x^{k-1})$, and compute the interpolant $\varphi^{k-1}$ for $\phi^{k-1} \land O_f^{k-1}(x^{k-1})$ and $\varphi^k$. The soundness of this step is straightforward. 

After having the interpolant, we use it to relax the constraint associated with the neurons according to the certificate. In the following, we show how it is done for the interval abstraction domain (which is adopted in our implementation). In general, the interpolant that we obtain (using MathSAT~\cite{mathsat5}) is in the form of a conjunction of linear constraints. For each variable $x^i[j]$, through transformation, we can obtain a formula in form of $\bigwedge x^i[j] \bowtie \varphi^i_j$, which $\bowtie$ is $\geq$ or $\leq$. Let $[lw, up]$ be the interval associated with $x^i[j]$. We can relax the constraint associated with $x^i[j]$, i.e., increase $up$ or decrease $lw$, as much as possible as long as the above formula is satisfied. In practice, we adopt a heuristic to increase (or decrease) the upper (or lower) bound using a factor of $K$ only, i.e., $up$ is set to be $up + K * |up|$ (or $lw$ is set to be $lw - K * |lw|$) each time. This is because relaxing the intervals would (1) have the effect of strengthening the certificate and (2) may reduce the accuracy of the network after model clipping (since it gives too much room for the new parameters). For these reasons, by default, we set $K = 1$. 

Theoretically, we can keep applying the relaxation and clipping all neurons layer-by-layer until the input layer and restoring the certificate (as well as the property). However, doing that may reduce the accuracy of the retrained network, and thus we only apply the relaxation and clipping on the ``important'' neurons at the output layer to restore the property without affecting the network accuracy too much.


\begin{example}
\label{exp:inter}
    For our running example, the MNIST network's properties remain valid after the retraining, i.e., it remains locally robust with respect to the 25 samples. This is mainly due to the certificate-based data augmentation step (see detailed experimental settings and results in Section~\ref{sec:eval}). However, the ACAS Xu network's property 3 is broken after retraining, and thus model clipping is necessary. Assume that we attempt to fix the model using the above-mentioned approach (on the last dense layer).
    The values of $lw - min(w * x_i)$ and $up - max(w * x_i)$ are as follows.
    \begin{align*}
    lw - min(w * x_i) & = [-0.1272, 0.6206, 1.3362, 0.6066, 0.9983]\\
    up - max(w * x_i) & = [-1.1367, -2.4329, -0.0455, 0.1117, 0.2636]
    \end{align*}
    As we can see, unfortunately, all the neurons in this layer satisfy $lw - min(w * x_i) > up - max(w * x_i)$ and so we cannot find a new value of $b$ to restore the certificate. 

Continue with the above example of ACAS Xu network. 
We have
\begin{align*}
    \phi^{13} & \equiv -3.3497 \leq x^{13}[0] \leq 1.7532 \wedge -5.6635 \leq x^{13}[1] \leq -4.7525 ~\wedge \\
    & 2.0411 \leq x^{13}[2] \leq 2.3862 \wedge 0.4508 \leq x^{13}[3] \leq 1.9311 ~\wedge \\
    & 1.1933 \leq x^{13}[4] \leq 4.2636 \\
    \xi & \equiv x^{13}[0] < x^{13}[1] \vee x^{13}[0] < x^{13}[2] \vee
    x^{13}[0] < x^{13}[3] ~\vee \\
    & x^{13}[0] < x^{13}[4]
\end{align*}
where $\phi^{13}$ is the constraint on the last layer according to the certificate and $\xi$ is the property (i.e., the label is not 0). Using MathSAT, we obtain the interpolant $C \equiv x^{13}[2] \ge x^{13}[0] + 0.2880$. The interpolant means the certificate is still valid if we increase the upper bound of $x^{13}[2]$ and decrease the lower bound of $x^{13}[0]$. We then apply relaxation and set the upper bound of $x^{13}[2]$ to $2.3862 + 1 * |2.3862|$ (i.e., double the old value). Similarly, we set the lower bound of $x^{13}[0]$ to $-3.3497 - 1 * |-3.3497|$. With new intervals, we succeed in finding the new biases for $x^{13}[0]$ and $x^{13}[2]$. The network after fixing satisfies property 3 as expected. We also notice that using the above approach to fix the network may have an impact on its accuracy. In this example, the accuracy reduces slightly from 86.2\% to 85.2\%. 
\end{example}


{\color{black} \noindent \textbf{Soundness} As we explained in the previous sections, the regularization does not guarantee the properties will be preserved. The soundness of CCL resides in the re-verification and model clipping. The soundness of the re-verification is trivial. So, we focus on the model clipping and have the following theorem.

\begin{theorem}[Soundness]
All the remaining certificates after the model clipping are valid.
\end{theorem}

\begin{proof} 
For each certificate that remains after the model clipping, we have $\forall i,l.~\phi_i^l \land f^l(x^l) \implies \phi_i^{l+1}$ due to the clipping procedure (including the output layer), which guarantees the certificate (and so the corresponding property) to be valid.
\end{proof}
}

%% file: contents/5_eval.tex
\section{Evaluation} \label{sec:eval}
We have implemented the above-introduced CCL approach in the Socrates framework~\cite{pham2020socrates} (which is an open-source project for experimenting with all kinds of neural network analysis techniques), with a total of 1,100 lines of Python code (excluding external libraries as well as libraries offered by the framework). Our implementation uses DeepPoly as a verification engine to generate certificates in the form of intervals for each neuron and uses MathSAT to generate Craig interpolants. We remark that our approach can be configured to work with alternative verification engines.  

\subsection{Experimental Subjects}
To evaluate our approach, we conduct multiple experiments using standard models trained on four popular datasets, i.e., ACAS Xu, MNIST, CIFAR10, and Census. These are chosen because the corresponding models are popular subjects for neural network verification and are associated with different properties, i.e., reachability, robustness, and fairness. Table~\ref{tbl:models} summarizes some relevant details of the models, i.e., the number of neurons, the testing accuracy, the number of properties, the property type, and the value of the hyper-parameters (when they are relevant). \\

\noindent \emph{The ACAS Xu dataset} is used to train an automatic airborne avoidance system. 
In~\cite{katz2017reluplex}, the authors provide 45 neural networks trained on the ACAS Xu dataset and analyzed these neural networks against 10 safety properties. To evaluate CLL, we train initial models such that property 3 is satisfied and retrain the models subsequently. Note that we choose property 3 because it is satisfied by most of the pre-trained networks. We use the same network architecture as in the pre-trained networks. The training process takes 200 epochs and we choose the network with the highest accuracy. In the end, we obtain 45 networks, 42 of which are certified with the property 3, exactly like in~\cite{katz2017reluplex}. We then use the precondition of property 10 to generate new training data. This property is chosen because it has a post-condition which is the opposite of that of property 3. We then retrain the 42 networks which are certified with property 3 with these new data. The retraining uses 20 epochs and we choose the networks with the highest accuracy as the final networks.\\

\noindent \emph{The MNIST dataset} contains handwritten greyscale digit images from 0 to 9. We divide the training set (and the test set) into two groups based on the labels, i.e., the first group contains the images with labels from 0 to 4 and the second one contains the rest of the images. We then train a network using the first group of images. The network is a fully-connected one that has three hidden layers. Each hidden layer has 25 neurons and the output layer has five neurons. 
Once the network is trained, we verify the robustness of the network against 25 randomly selected samples in the test set and obtain the certificates accordingly. The robustness is verified based on an $L_\infty$-norm with an $\epsilon$ of $0.01$. To retrain the neural network with the group of images with 5 to 9, we extend each of the hidden layers with 25 new neurons and the output layer with 5 new neurons and apply CCL. The goal is to evaluate whether the retrained neural network remains robust with respect to the 25 samples and furthermore whether the network has high accuracy in recognizing all 10 digits. For both the initial training and the retraining, we use 20 epochs and choose the network with the highest accuracy. \\

\noindent \emph{The CIFAR10 dataset} contains colored images, which show objects belonging to one of the 10 classes. We similarly divide the CIFAR10 dataset into two groups, i.e., the first one contains the first five classes of objects and the second one contains the rest. We use the same network architecture as the one for the MNIST dataset, except each hidden layer has 250 neurons and is extended with 250 new ones in the retraining. We first train a network with the training data from the first group. Afterward, we apply DeepPoly to obtain the certificates for the local robustness of the neural network with respect to 25 samples randomly selected from the test set. Lastly, we retrain the network with the second group. For both the initial training and the retraining, we use 50 epochs and choose the network with the highest accuracy. Due to the high-resolution images in CIFAR10, the neural network obtained this way is expectedly less accurate and less robust than the MNIST neural networks. Thus, we use an $L_\infty$-norm of $0.001$ for the local robustness verification. \\

\noindent \emph{The Census dataset} contains features regarding certain information about adult individuals. The feature gender is considered a sensitive feature. The task is to train a binary classifier that predicts whether an individual owns more than USD 50,000 annually. We adopt the same model used in the previous studies~\cite{zhang2020white}. We train a network using all training samples, and then randomly select 10 samples from the test set to obtain certificates on fairness. To obtain a set of new data to retrain the neural network, we apply an existing fairness testing engine ADF~\cite{zhang2020white} to generate individual discriminatory instances, and retrain the network with them. The goal of the study is to evaluate whether fairness properties that are satisfied by the original neural network might be violated after such retraining and whether CCL can be applied to address the problem. For both the initial training and the retraining, we use 20 epochs and choose the network with the highest accuracy. \\

\begin{table*}[t]
\begin{center}
\resizebox{0.6\textwidth}{!}{
\begin{tabular}{| l | c | c | c | c | c | c | c |}
\hline
 Models & \#Neurons & Learn. Rate & Acc. (\%) & \#Prop. & Prop. Type & $\alpha$, $\beta$ \\
 \hline
 ACAS Xu & 310         & 0.01 & 89.4 & 42 & Reachability & 0.001  \\ 
 MNIST   & 864-944     & 0.01 & 98.8 & 25 & Robustness   & 0.001  \\ 
 CIFAR10 & 3827-4582   & 0.01 & 65.3 & 25 & Robustness   & 0.001  \\
 Census  & 139         & 0.1 & 84.8 & 10 & Fairness      & 0.001  \\
 \hline
\end{tabular}
}
\caption{Details of the ACAS Xu (42 models), MNIST, CIFAR10, and Census models}
\label{tbl:models}
\end{center}
\end{table*}

\noindent \emph{Baselines} Since our approach is the first method to support certified continual learning, we compare it with the following baseline approaches in our experiments.
\begin{itemize}
    \item \emph{Baseline}: This is the approach that applies ordinary continual learning (without any of the original training data).
    \item \emph{Baseline+DS}: To understand the relevance of certificate-based data augmentation, we evaluate the effect of simply applying certificate-based data augmentation without applying the remaining components of CCL using the approach called \emph{Baseline+DS}. To keep the number of synthesized data small, we synthesize 5 samples for the ACAS Xu models; and 10 samples for each certificate for the remaining models (e.g., with the MNIST model, we have 25 robust certificates and thus a total of 250 samples are synthesized). Intuitively, few samples are synthesized for the ACAS Xu models because the models and the property are considered simpler compared to the rest.
    \item \emph{Baseline+OD}: This is the approach that applies ordinary continual learning (with 20\% of the original training data). Note that this is one optimistic continual learning setting that was adopted in some existing approaches~\cite{shin2017continual,van2018generative,rebuffi2017icarl,nguyen2017variational,kemker2017fearnet}.
    \item \emph{CCL}: This is the approach where CCL as described in the previous sections is applied. 
    \item \emph{CCL+OD}:  This is the approach where 20\% of the original training data is assumed to be available and is used together with CCL.
\end{itemize}

\subsection{Research Questions and Answers}
In the followings, we answer multiple research questions (RQs) based on the experimental results. All experiments are run on a machine with 2.6 GHz 6-Core CPU and 16 GB. We did not use GPU in the experiments. For reproducing the experimental results, all the trained models, datasets, verified properties, and generated certificates are made publicly available at~\cite{ccl}. \\

\begin{table*}[t]
\begin{center}
\resizebox{0.8\textwidth}{!}{
\begin{tabular}{| l | c | c | c | c | c | c | c | c | c | c | c |}
\hline
 \multirow{2}{*}{Models} & \multirow{2}{*}{OAcc. (\%)} & \multicolumn{2}{c|}{Baseline} & \multicolumn{2}{c|}{Baseline+DS} & \multicolumn{2}{c|}{Baseline+OD} & \multicolumn{2}{c|}{CCL} & \multicolumn{2}{c|}{CCL+OD} \\
 \cline{3-12}
 & & Prop. & Acc. (\%) & Prop. & Acc. (\%) & Prop. & Acc. (\%) & Prop. & Acc. (\%) & Prop. & Acc. (\%) \\
 \hline
 ACAS Xu & 89.6 & 37 & 89.3 & 41 & 89.5 & 42 & 90.8 & 42 & 89.4 & 42 & 90.7  \\ 
 MNIST   & 98.8 & 0  & 46.9 & 25 & 75.3 & 23 & 93.0 & 25 & 74.7 & 25 & 93.2  \\ 
 CIFAR10 & 65.3 & 0  & 37.0 & 25 & 29.6 & 4  & 46.8 & 25 & 33.0 & 25 & 39.8  \\
 Census  & 84.8 & 10 & 81.9 & 10 & 83.5 & 10 & 84.1 & 10 & 82.6 & 10 & 84.1  \\
 \hline
\end{tabular}
}
\caption{Results on the ACAS Xu (42 models), MNIST, CIFAR10, and Census models}
\label{tbl:other}
\end{center}
\end{table*}

\noindent \emph{RQ1: Does CCL preserve the properties after retraining a neural network?} To answer this question, we systematically apply the above-mentioned five approaches to all the models. 
The results are summarized in Table~\ref{tbl:other}, where column \emph{Model} shows the name of the dataset that the model is trained on, column \emph{OAcc.} shows the accuracy of the original network before retraining, column \emph{Prop.} shows the number of properties that are preserved by the retraining and column \emph{Acc.} shows the accuracy of the network after retraining.  
Note that the row named \emph{ACAS Xu} is the combined results of 42 models with the average accuracy.

There are multiple observations that we can make based on the results. First, it can be observed that indeed the verified properties are likely to fail after retraining without applying CCL, i.e., for the ACAS Xu models, 5 out of the 42 models fail to preserve the verified property; and none of the 50 properties for the MNIST and CIFAR10 are preserved. The only exception is the Census model, where all 10 properties are preserved. We conjecture that how likely it happens depends mainly on the distribution of the new data. That is, the more different the new data distribution is, the more catastrophic forgetting occurs and thus the properties are violated. This is most observable in the CIFAR10 model (where the new data follows a completely different distribution from that of the original one), whose accuracy drops significantly and all of its certificates are broken after retraining. On the contrary, most of the properties are preserved in the case of the ACAS Xu models and the Census model as the new data distribution is similar to the original one. 

Second, it can be observed that the certificates can help even if only a lightweight approach such as certificate-based data augmentation is applied. The results of \emph{Baseline+DS} show that almost all properties can be successfully re-verified when certificate-based data augmentation is applied. Comparing the results of \emph{Baseline+DS} and \emph{Baseline+OD}, it can be observed that certificate-based data augmentation works much better than having part of the original training data in terms of preserving the properties. This is because the original training data rarely satisfies the precondition required by the properties, especially when the new data follows a very different distribution.

Thirdly, CCL successfully preserves all the properties in our experiments. Furthermore, CCL achieves that mostly based on lightweight approaches such as data augmentation and regularization. Only for one of the ACAS Xu models, certificate relaxation and model clipping are applied successfully to restore the property. We remark that if certificate-based data augmentation is not applied, model clipping is likely applied much more frequently. Our suggestion, however, is to apply certificate-based data augmentation whenever possible so that the heavier model clipping is applied only occasionally. 

While the objective of CCL is to preserve the properties, it can be observed that it has some positive effect on the model accuracy as well, i.e., it helps to improve the accuracy of the MNIST model significantly, and improve the accuracy of the Census model slightly. Although it seems to reduce the accuracy of the CIFAR10 model, our investigation shows that \emph{Baseline} achieves a better accuracy only because it completely forgets about the original training set in this case (and it so happens that the model has low accuracy in general), whereas CCL prevents it from doing so. Overall, the effect of CCL improving accuracy is relatively minor compared to the two dominating factors, i.e., whether the new data distribution is similar to that of the original and whether some part of the original training data is available. Contrasting the results on the ACAS Xu models and the Census model (i.e., the distribution of the original and new data is similar) with those of the MNIST and CIFAR10 models (i.e., the distribution of the original and new data is different), we can observe that the more similar the new data's distribution is to the original training data, the less catastrophic forgetting there is and thus the less the accuracy drop is. 
Contrasting the results of \emph{Baseline} and \emph{Baseline+OD}, as well as \emph{CCL} and \emph{CCL+OD}, we can observe that having some original data definitely improves the accuracy. \\

\begin{table*}[t]
\begin{center}
\resizebox{0.52\textwidth}{!}{
\begin{tabular}{| l | c | c | c | c | c |}
\hline
 Models & Baseline & Baseline+DS & Baseline+OD & CCL & CCL+OD \\
 \hline
 ACAS Xu & 1m17s & 1m19s & 1m29s & 1m24s & 1m32s \\
 MNIST   & 1m12s & 1m11s & 1m23s & 1m13s & 1m24s \\
 CIFAR10 & 7m26s & 7m35s & 8m39s & 7m57s & 9m7s \\
 Census  & 0m3s & 0m2s & 0m4s & 0m3s & 0m4s \\
 \hline
\end{tabular}
}
\caption{Training time on ACAS Xu (42 models), MNIST, CIFAR10, and Census models}
\label{tbl:time}
\end{center}
\end{table*}

\noindent \emph{RQ2: Is the overhead of CCL significant?} To answer this question, we record the training time with and without applying CCL. The results are shown in Table~\ref{tbl:time}.
We observe the same pattern across all the models. The \emph{Baseline} and the \emph{Baseline+DS} approaches have the least training time, since they use the ordinary training method and have the least training data (recall that only a handful of samples are synthesized for each certificate). The \emph{CCL} approach spends more time on training than the \emph{Baseline} and the \emph{Baseline+DS} approaches due to its more complex training method.
Note that the \emph{CCL} approach uses less training time than the \emph{Baseline+OD} approach because the latter uses much more training data (e.g., for the CIFAR10 dataset, the 20\% of original training data is about 5,000 samples while synthesized samples are only 250). The results imply that the amount of training data has a more noticeable impact on the training time than CCL.
Finally, as expected, the \emph{CCL+OD} approach uses the most training time. Based on the results, we conclude that CCL's overhead is negligible in practice.
 \\

\noindent \emph{RQ3: How does the certificate relaxing and model clipping affect the accuracy of the model?}
Recall that certificate relaxing and model clipping are applied only when it is necessary, i.e., the property is violated after training, in which case we update the model parameters slightly based on a hyperparameter $K$ to restore the correctness of the model. Updating the model parameters certainly would have an impact on the model's accuracy. We thus would like to check whether the impact is severe. Fortunately (thanks mainly to certificate-based data augmentation), in all of our experiments (with a total of 45 models), there is only one model that needs to be clipped. The details of the case have already been discussed in Example~\ref{exp:inter}. In this experiment, we show how much the accuracy of the network is affected by different values of $K$. We adopt different values of $K$ and evaluate their impact. The results are summarized in Fig.~\ref{fig:k}.

\begin{figure*}[t]
\centering
\includegraphics[scale=0.05]{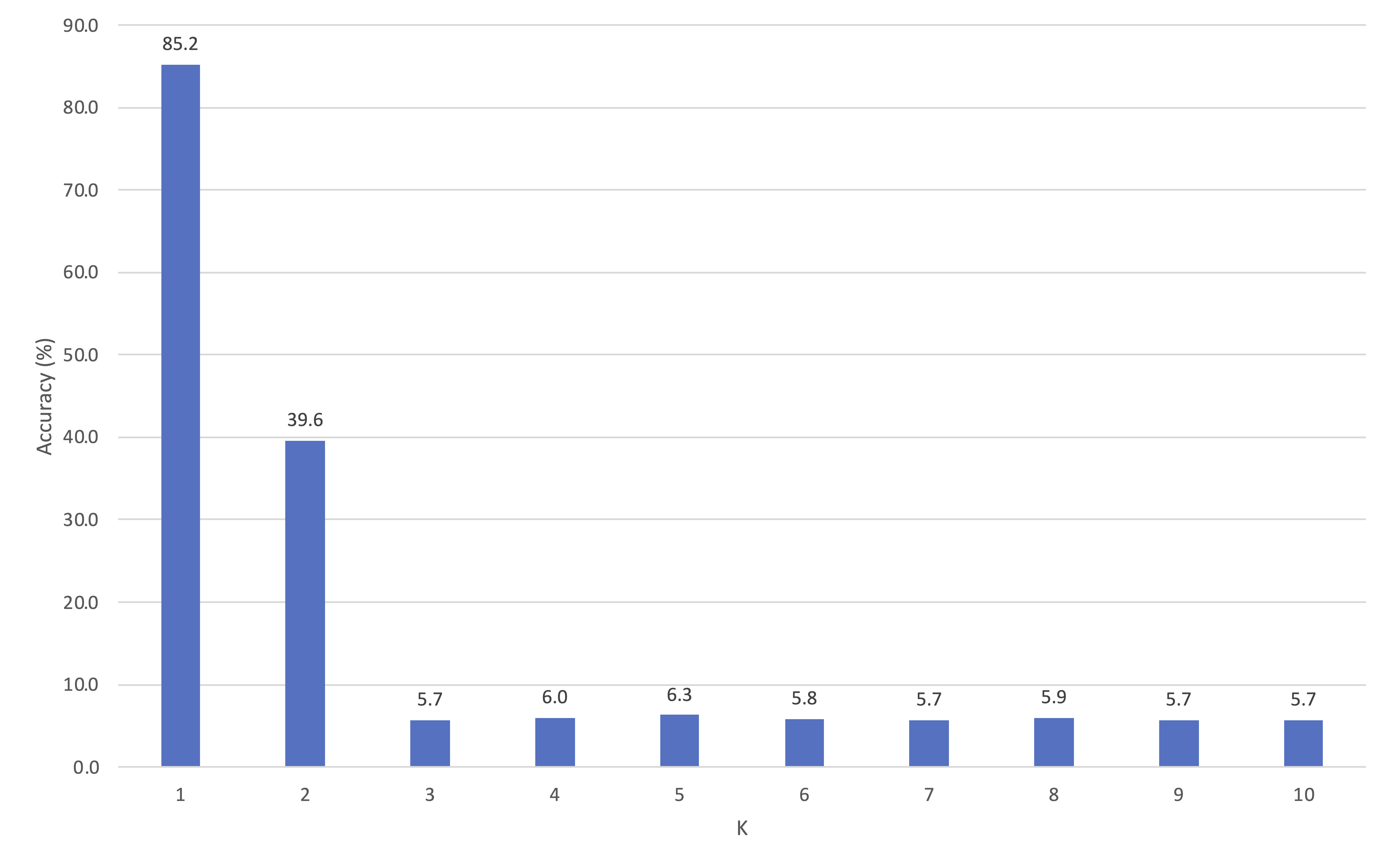}
\caption{The accuracy on different values of $K$}
\label{fig:k}
\end{figure*}


Although the number of cases is small, from the table, we can see that as the value of $K$ increases, the accuracy of the network decreases rapidly and drastically, and stabilizes after reaching 3. The results clearly show that, as expected, when the value of $K$ increases, the intervals of the neurons also enlarge. As a result, a huge change in the biases (or weights) may be introduced through model clipping, and consequently reduce the accuracy of the network. This result thus suggests that our strategy of gradually increasing the value of $K$, i.e., relaxing the certificate, and stopping as soon as the broken properties are fixed is a reasonable one. \\

\noindent \emph{RQ4: Is CCL effective if the network is retrained multiple times?} In practice, it is entirely likely that the neural network must be retrained multiple times. To evaluate whether CCL is relevant in such a usage scenario, we conduct additional experiments based on a model trained on the MNIST dataset and another one trained on the CIFAR10 dataset. In both cases, we divide the dataset into five groups, each of which contains training data of two classes. For instance, the five groups of the MNIST dataset contain images labeled 0/1, 2/3, 4/5, 6/7, and 8/9 respectively. We then train an initial neural network using the first group of data and then retrain it with the remaining groups one by one. In total, there are four times of retraining. Before each round of training, we add 10 neurons to each hidden layer for the MNIST model and 100 neurons for the CIFAR10 model. In each round, two neurons are added to the output layer.
Each group of data is trained with 20 epochs for the MNIST model and 50 epochs for the CIFAR10 model. After training with the first group of data, we randomly select 25 samples from the test set and use DeepPoly to verify the local robustness of the model with regard to these 25 samples and obtain the corresponding certificates.
The idea is to check whether these properties can be preserved through multiple retraining. 

The results are shown in Table~\ref{tbl:mnist} and Table~\ref{tbl:cifar10}. First, we can observe that catastrophic forgetting occurs during each retraining for both models, e.g., with the \emph{Baseline} approach, the final model has an accuracy of less than 20\%, which means that the model has completely forgotten about all the previous knowledge. Second, with the \emph{Baseline+OD} approach, the accuracy is much improved, although there is a steady decline along the way. Besides accuracy, the certificates may become invalid and sometimes magically become valid again. More importantly, if a property is found to be violated, it is challenging to repair the model such that the property is restored. Third, with the lightweight approach of certificate-based data augmentation (\emph{Baseline+DS}, \emph{CCL}, and \emph{CCL+OD}), we are able to preserve the properties through four times of retraining.
Forth, comparing the accuracy of the \emph{Baseline} approach and the \emph{CCL} approach, we can observe that CCL helps to improve the accuracy as well. Lastly, comparing the accuracy of the \emph{Baseline+OD} approach and the \emph{CCL+OD} approach, it can be observed that the former is likely to have better accuracy, especially in the case of the CIFAR10 model. This is likely due to the certificate-based regularizer, which constrains how the network parameters are updated, i.e., a price that we have to pay to preserve the model's correctness.

{\color{black} Finally, to get more insights about the performance of CCL, we collect the accuracy of different labels for the last model trained by CCL in Table 4. In general, the accuracy is divided into 3 groups:

\begin{itemize}
\item Group 1 includes samples with labels 0 and 1. Thanks to the data augmentation step and the constraint in the loss function, we maintain a relatively good accuracy (64.4\% for label 0 and 78.1\% for label 1).
\item Group 2 includes samples with labels 8 and 9. These are the newest training data and so they have the highest accuracy (95.2\% for label 8 and 96.9\% for label 9).
\item Group 3 includes samples with labels from 2 to 7. There are no constraints or training data during the last round, catastrophic forgetting happens and the accuracy for these samples drops significantly, i.e., to 0.0\%.
\end{itemize}
 
The results are consistent with our understanding. We also notice that in real applications, new training data may come with new properties. With the certificates for these properties, most of the samples will fall into Group 1 instead of Group 3. Moreover, some approaches to keep part of old training data may help. For instance, using the CCL+OD setting, the last model has accuracy ranging from 74.8\% (label 5) to 97.9\% (label 1).\\
}

{\color{black} \noindent \emph{Limitation.} In the following, we discuss some limitations of our approach.

First, the performance of the model may decline through multiple rounds of retraining (due to the constraints introduced by the accumulated certificates as well as the difference between the new data and the old data). In the extreme case, the new data and old data may even contradict each other which makes the accuracy to be very low or the certificates/properties impossible to satisfy. Practically, the users should choose whether to only finetune the existing neurons, introduce new neurons, or train a new model from scratch, based on the accuracy of the model. If the accuracy is too low or CCL complains that the certificates cannot be repaired, the users will have to train a new model from scratch.

Second, as we show in Table 3, the data augmentation and regularization step of CCL has no scalability issue. The verification step however does and it is safe to say that our scalability is similar to the scalability of existing neural network verification techniques. Note that we cannot avoid this scalability issue if we would like formal guarantees on certain properties. To use CCL for larger models, more resources are needed to overcome the scalability issue in the verification step.
}

\begin{table*}[t]
\begin{center}
\resizebox{0.8\textwidth}{!}{
\begin{tabular}{| l | c | c | c | c | c | c | c | c | c | c |}
\hline
 \multirow{2}{*}{No of classes} & \multicolumn{2}{c|}{Baseline} & \multicolumn{2}{c|}{Baseline+DS} & \multicolumn{2}{c|}{Baseline+OD} & \multicolumn{2}{c|}{CCL} & \multicolumn{2}{c|}{CCL+OD} \\
 \cline{2-11}
 & Cert. & Acc. (\%) & Cert. & Acc. (\%) & Cert. & Acc. (\%) & Cert. & Acc. (\%) & Cert. & Acc. (\%) \\
 \hline
 $2$  & 25 & 100.0 & 25 & 100.0 & 25 & 100.0 & 25 & 100.0 & 25 & 100.0   \\ 
 $4$  & 0 & 48.1   & 25 & 88.1  & 22 & 97.4  & 25 & 91.1  & 25 & 97.8   \\
 $6$  & 0 & 30.9   & 25 & 56.8  & 24 & 93.3  & 25 & 60.1  & 25 & 93.7   \\
 $8$  & 0 & 24.6   & 25 & 44.1  & 24 & 91.7  & 25 & 44.8  & 25 & 92.3   \\
 $10$ & 0 & 19.5   & 25 & 33.4  & 25 & 90.1  & 25 & 34.2  & 25 & 89.5   \\
 \hline
\end{tabular}
}
\caption{Multiple training results on the MNIST model}
\label{tbl:mnist}
\end{center}
\end{table*}

\begin{table*}[t]
\begin{center}
\resizebox{0.8\textwidth}{!}{
\begin{tabular}{| l | c | c | c | c | c | c | c | c | c | c |}
\hline
 \multirow{2}{*}{No of classes} & \multicolumn{2}{c|}{Baseline} & \multicolumn{2}{c|}{Baseline+DS} & \multicolumn{2}{c|}{Baseline+OD} & \multicolumn{2}{c|}{CCL} & \multicolumn{2}{c|}{CCL+OD} \\
 \cline{2-11}
 & Cert. & Acc. (\%) & Cert. & Acc. (\%) & Cert. & Acc. (\%) & Cert. & Acc. (\%) & Cert. & Acc. (\%) \\
 \hline
 $2$  & 25 & 88.7 & 25  & 88.7 & 25 & 88.7 & 25 & 88.7 & 25 & 88.7   \\ 
 $4$  & 0  & 38.9 & 25  & 53.3 & 18 & 67.9 & 25 & 52.2 & 25 & 57.4   \\
 $6$  & 0  & 26.9 & 25  & 39.4 & 15 & 48.6 & 25 & 40.6 & 25 & 45.7   \\
 $8$  & 0  & 22.3 & 25  & 31.5 & 17 & 46.1 & 25 & 30.4 & 25 & 38.6   \\
 $10$ & 0  & 17.4 & 25  & 19.5 & 8  & 45.8 & 25 & 18.2 & 25 & 33.8   \\
 \hline
\end{tabular}
}
\caption{Multiple training results on the CIFAR10 model}
\label{tbl:cifar10}
\end{center}
\end{table*}

%% file: contents/6_related.tex
\section{Related Work} \label{sec:related}
This work is closely related to works on neural network verification and works on continual learning. We already discuss these lines of work in Section~\ref{sec:background}. In the following, we discuss some broadly related work on improving the likelihood of neural networks satisfying desirable properties. \\

\noindent \emph{Adversarial training} is a group of methods that aim to improve a neural network's robustness by training additionally with adversarial samples (either generated before training or during training)~\cite{fgsm,kurakin2016adversarial,madry2017towards,balunovic2020adversarial}. The state-of-the-art approaches in this category transform the training process into a min-max optimization problem, which is defined as follows.

$$
\argmin_{\theta} \max_{(x,y) \in D,~x' = A(x)} \sum{loss(N_{\theta}(x'), y)}
$$
where $\theta$ is the set of parameters (e.g., weights and biases) of the network $N$; $D$ is the training set, which contains a set of tuples with input vector $x$ and output vector (or label) $y$; $x'$ is an adversarial sample generated based on $x$ using the attacking function $A$. Note that typically $x'$ should satisfy $d(x',x) \leq c$. Intuitively, the inner max part is the maximal damage that can be caused by a perturbation and the goal is to minimize the maximal damage. Different approaches use different attacking functions to find $x'$ as well as different methods to approximate the inner max part and solve the above optimization problem (e.g., using gradient descent and updating all the parameters at once~\cite{fgsm,kurakin2016adversarial,madry2017towards} or only updating layer by layer~\cite{balunovic2020adversarial}). Such an approach can be potentially extended to support other properties as well. While the above adversarial training approaches are shown to improve models' robustness, they cannot provide a certificate (except~\cite{balunovic2020adversarial}).
Instead, the improvement is measured by attacking the trained model using some existing adversarial attacks (e.g., C\&W~\cite{carlini2017towards}). \\

\noindent \emph{Certified training} aims to train neural networks in a way such that it is guaranteed that the network satisfies some properties (e.g., robustness or fairness). The general idea is to use a sound approach to compute an upper bound of the inner max part, and optimize the network parameters so that the maximal damage is soundly limited to a certain degree and as a result, the neural network can be certified against certain properties. Approaches in this category differ in terms of how they solve the inner max part. Example approaches include using Lipschitz regularization~\cite{hein2017formal}, linear programming~\cite{wong2018provable}, semidefinite relaxation~\cite{raghunathan2018certified,dvijotham2018training}, hybrid zonotope~\cite{mirman2018differentiable}, or interval bound~\cite{gowal2018effectiveness}. Another line of approach to train certified networks (against robustness or privacy) is referred to as randomized smoothing or differential privacy. In~\cite{lecuyer2019certified}, Lecuyer \emph{et al.} propose to add a noise layer into the model architecture and shows that it can help to achieve certified robustness. In~\cite{DBLP:conf/ccs/AbadiCGMMT016}, Abadi \emph{et al.} propose the \emph{differentially private SGD} training algorithm. The idea is that random noise is introduced in a principled manner in the backpropagation step before updating the model's parameters. It is shown that their approach can be proved to satisfy differential privacy. \\

\noindent \emph{Network repairing} is a line of approach that are complementary to ours, which aims to modify a trained model so that it satisfies certain desired properties. In this line of work, CARE~\cite{sun2022causality} is the state-of-the-art, in which, the authors use causal analysis to identify the neurons which are most responsible for violating the property. The neural network is then repaired by solving an optimization problem using a fitness function in which the weights and biases of the responsible neurons are the variables. The goal is to find new values for the weights and biases of the responsible neurons such that the repaired network satisfies the desired properties and still maintains high accuracy. Other existing works in this category include~\cite{sotoudeh2019correcting,goldberger2020minimal,usman2021nn,sotoudeh2021provable}, which differ, in essence, in the way of finding responsible neurons/layers and identifying new values for their weights and biases. For instance, Sotoudeh \emph{et al.}~\cite{sotoudeh2019correcting} solve the problem based on SMT; Goldberger \emph{et al.}~\cite{goldberger2020minimal} solve the problem based on neural network verification techniques; Usman et al.~\cite{usman2021nn} rely on constraint-based solving; and Sotoudeh et al.~\cite{sotoudeh2021provable} use linear programming. Finally, NeuRecover \cite{tokui2022neurecover} is loosely related to this line of work and uses particle swarm optimization to suppress regression errors in neural networks.\\

\noindent Compared to the above-mentioned related work, CCL is based on neural network verification and continual learning and aims to address the problem of maintaining verification certificates as long as possible during continual learning. To the best of our knowledge, this problem has not been addressed yet.

%% file: contents/7_conclusion.tex
\section{Conclusion} \label{sec:conclusion}
In this work, we proposed a method called certified continual learning, which aims to preserve the correctness of a verified neural network as much as possible when it is retrained through continual learning. The main idea is to bundle the model with its correctness certificates and use the certificates to guide how and where the neural network should be updated during retraining. The experimental results suggest our approach can effectively preserve the correctness of a neural network and often has a positive impact on the model accuracy as well.  

For future works, we plan to explore whether our approach can support probabilistic correctness properties such as group fairness~\cite{sun2021probabilistic}, backdoor-freeness~\cite{pham2022backdoor}, and differential privacy~\cite{DBLP:conf/ccs/AbadiCGMMT016}.